\RequirePackage{fix-cm}
\documentclass[smallcondensed,natbib]{svjour3}     
\smartqed  
\usepackage{graphicx}

\usepackage{url}
\usepackage{color}
\usepackage{graphicx}
\usepackage{amssymb,amsfonts}
\usepackage{amsmath,amssymb}
\usepackage{float}
\usepackage{comment}
\usepackage{verbatim}
\usepackage{appendix,stfloats}
\usepackage{bbm}
\usepackage{blkarray}

\usepackage{wrapfig}
\usepackage{epsfig}
\usepackage{amsmath,amssymb,graphicx,url,mathtools}
\usepackage{algorithm}
\usepackage{algorithmic}
\usepackage{hyperref}
\usepackage{enumerate}
\usepackage{subfigure,color}
\usepackage{array,dsfont}

\usepackage{wrapfig}

\newcommand{\argmin}{\operatornamewithlimits{argmin}}

\def \bg{\mathbf g}

\newtheorem{thm}{Theorem}[section]
\newtheorem{prop}[thm]{Proposition}

\newtheorem{cor}[thm]{Corollary}

\def \cD {\mathcal{D}}

\def \bal {\begin{align}}
\def \eal {\end{align}}

\def \bx {\mathbf{x}}

\newcommand{\indicator}[1]{\mathds{1}_{\left[ {#1} \right] }}

\usepackage{graphicx}

\begin{document}

\title{Sensor Selection by Linear Programming
}


\author{Joseph Wang         \and
        Kirill Trapeznkov \and Venkatesh Saligrama 
}


\institute{Joseph Wang \at
              8 St. Mary's Street, Boston, MA 02215, USA\\
              Tel.: (617) 353-2811\\
              Fax: (617) 353-7337\\
              \email{joewang@bu.edu}           
           \and
           Kirill Trapeznikov \at
              \email{kirill.trapeznikov@stresearch.com}
			\and
            Venkatesh Saligrama \at
            \email{srv@bu.edu}
}

\date{Received: date / Accepted: date}

\maketitle

\begin{abstract}
We learn sensor trees from training data to minimize sensor acquisition costs during test time. Our system adaptively selects sensors at each stage if necessary to make a confident classification. We pose the problem as empirical risk minimization over the choice of trees and node decision rules. We decompose the problem, which is known to be intractable, into combinatorial (tree structures) and continuous parts (node decision rules) and propose to solve them separately. Using training data we greedily solve for the combinatorial tree structures and for the continuous part, which is a non-convex multilinear objective function, we derive convex surrogate loss functions that are piecewise linear. The resulting problem can be cast as a linear program and has the advantage of guaranteed convergence, global optimality, repeatability and computational efficiency. We show that our proposed approach outperforms the state-of-art on a number of benchmark datasets.
\end{abstract}

\section{Introduction}
%
%
%
Many scenarios involve classification systems constrained by measurement acquisition budget. In this setting, a collection of sensor modalities with varying costs are available to the decision system. Our goal is to learn adaptive decision rules from labeled training data that when presented with a new unseen example would select the most informative and cost-effective strategy for the example. In contrast to non-adaptive methods \cite{efron2004least,xu2012greedy}, which attempt to identify a common sparse subset of sensors that can work well for all examples, our goal is an adaptive method that can classify typical cases using inexpensive sensors and using expensive sensors only for atypical cases.

We propose to learn a sensor tree using labeled training examples for making decisions on unseen test examples. The learned sensor tree is composed of internal node decision rules. Given an example these decision rules select sensors and guide an example along a particular path terminating at a leaf where it is classified with a classifier. We pose the problem as a global empirical risk minimization (ERM) over the choice of tree structures, node decision rules and leaf classifiers.

The general problem is a highly coupled problem consisting of combinatorial (sensor tree structure) and continuous components (decision rules to generalize to unseen examples) and difficult to optimize. To gain further insight we abstract away the generalization aspect and observe that the resulting combinatorial problem, a special case of ours, is known to be NP hard and requires greedy approximations~\cite{chakr07,cicalese}.


The combinatorial issue can be circumvented in cases where expert knowledge exists, only a few sensors (as in \cite{trapeznikov:2013,trapeznikov:2013b,wang2014lp}) exist or for small depth trees. For these latter two cases we contruct an exhaustive tree and globally learn decision functions using a linear program by generalizing the cascade structures presented in \cite{trapeznikov:2013b,wang2014lp} to binary trees, resulting in more flexible decision systems. Convex surrogates of products of indicators have previously been studied for supervised learning \cite{wang2013locally}.

For more general cases we propose a two-step approach to decouple the issue of sensor structure from the decision rule design. 
\begin{itemize}
\item We greedily solve for the combinatorial tree structure and obtain feature/sensor sub-collections efficiently by a greedy approximation to the NP-hard problem. From these subsets, we construct a binary tree using hierarchical clustering of feature subsets.
\item On the learned tree structure, our problem now reduces to the ERM problem discussed above for a fixed tree where we apply a novel surrogate, allowing us to jointly learn the decision functions in the tree by solving a linear program.
\end{itemize}

In the experiments, we demonstrate performance of our approach both for when feature subsets and tree structure is given and when feature subsets and tree structure must be learned. We show on real world data that our approach outperforms previously proposed approaches to budgeted learning.

\subsection{Related Work}
There is an extensive literature on adaptive methods for sensor selection for reducing test-time costs. It arguably originated with detection cascade (see \cite{zhang:2010,chen:2012} and references therein), which is a popular method in reducing computation cost in object detection for cases with highly skewed class imbalance and generic features. Computationally cheap features are used at first to filter out negative examples and the more expensive features are used in the later stages.

Our technical approach is closely related to \cite{trapeznikov:2013b} and \cite{wang2014lp}. Like us they formulate an ERM problem and generalize detection cascades to classifier cascades and handle balanced and/or multi-class scenarios. Like us, \cite{wang2014lp} construct convex surrogates for their empirical risk functions and propose efficient LP solutions. Unlike us their approach is limited to cascades of known structure and cannot handle trees and unknown sensor structures. 

Conceptually, our work is closely related to ~\cite{xu2013cost} and \cite{kusner2014feature}, who introduced cost-sensitive tree of classifiers (CSTC) for reducing test time costs. Like our paper they proposed a global ERM problem. They solve for the tree structure, internal decision rules and leaf classifiers jointly using alternative minimization techniques. Recently, \cite{kusner2014feature} propose a more efficient version of CSTC. In contrast we decompose our global objective and separately solve the individual parts. The disadvantage of our decoupled approach is somewhat offset by globally convergent solution for the decision rules once a structure is determined. Nevertheless, which approach is better is an important question that must be addressed but outside the scope of this work.

\nocite{wang2014model}



The subject of this paper is broadly related to other adaptive methods in the literature but unlike us these methods do not learn sensor trees but learn policies from training data. Generative methods pose the problem as a POMDP, learn conditional probability models \cite{zubek:2002,Sheng06featurevalue,bilgic:2007,ji:2007,kanani:2008,kapoor:2009,gao2011active}
and myopically select features based information gain of unknown features.
MDP-based methods ~\cite{karayev2013dynamic},  ~\cite{dulac2011datum}, \cite{he2012imitation}, \cite{busa2012fast}
encode current observations as state, unused features as action space, and formulate various reward functions to account for classification error and costs.
He et. al. \cite{he2012imitation} apply imitation learning of a greedy policy with a single classification step as actions. ~\cite{dulac2011datum} and \cite{karayev2013dynamic} apply reinforcement learning to solve this MDP. \cite{busa2012fast} propose classifier cascades within an MDP framework. They consider a fixed-ordering of features and extend sequential boosted classifier with an additional skip action.

\section{Problem Formulation: Global Empirical Risk Minimization Objective}\label{sec.struct}
We consider learning an adaptive decision system with training examples $(x_1,y_1),\allowbreak \ldots,\allowbreak(x_N,y_N)$ with $L$ sensors and acquisition cost $c_m,\,m=1,\,2,\ldots,L$. We can pose the problem of learning a rooted sensor tree as an ERM problem.
Our system is composed of three components, a tree, $T$, decision rules $\mathbf{g}_J=\{g_j\}_{j=1}^J,\, \bg_j \in {\cal G}_j$ associated with $J$ internal nodes, and classifier $\mathbf{f}_K=\{f_k\}_{k=1}^K,\,f_k \in {\cal F}_k$ associated with leaves. Each internal node of the tree is associated with a sensor and its children denote available sensor choices. Each leaf, $k$, is associated with a sensor subset, $S_k$, and corresponds to the unique path from the root to the leaf. The decision rule, $g_j$, associated with node $j$ acts upon acquired measurements for an example and routes it to one of its children. By uniquely associating each leaf, $k$, with the sensor outputs $b_k$ we can write ERM as:
\begin{align}\label{eq.globalerm}
L({\cal S},\mathbf{f},\mathbf{g})=\sum_{i=1}^{N}\sum_{k=1}^{K}\left(\mathbbm{1}_{f_k(x_i)\neq y_i}+\alpha\sum_{m \in S_k}c_m\right)\mathbbm{1}_{\mathbf{g}_J(x_i)=b_k}\notag\\ \longrightarrow \min_{{\cal S}} \min_{\mathbf{f}_K} \min_{\bg_J} L({\cal S},\mathbf{f},\mathbf{g})
\end{align}
where $\alpha$ is a trade-off parameter balancing classification performance with sensor acquisition cost and ${\cal S}$ is the set of paths. An instance of this general problem has been considered in the literature and shown to be NP hard~\cite{chakr07} for the special case of discrete valued sensor measurements, arbitrarily powerful decision rules, $\bg_J$ and with separable leaves, i.e, features acquired corresponding to the leaf path uniquely \& correctly identifies the class. The authors develop greedy algorithms for approximating their solution.
In our setting we allow for continous valued high-dimensional measurements and so we cannot discretize our space. Furthermore we are not in a separable situation and the Bayes error is not zero. However, our general ERM problem is highly coupled, difficult to optimize. This motivates our proposed decomposition approach described below.

Indeed, assuming arbitrary families ${\cal G}_j$ in Eq.~\ref{eq.globalerm} leads to new insights that motivates our decomposition approach. The general problem reduced to a purely combinatorial structure learning with arbitrary ${\cal G}_j$ when we also have access to an oracle classifier $\mathbf{f}_K$ capable of classifying with any subset of acquired features.
The resulting problem while NP hard as before is amenable to greedy strategies. Alternatively, given a tree structure and the oracle classifier the optimization objective takes a multilinear form as evidenced by Eq.~\ref{eq.globalerm} which also lends itself to optimization strategies. The only issue remaining is that of an oracle classifier, which can be circumvented for our approach.
This overall perspective justifies our approach: 
\begin{enumerate}
\item Learn the tree structure assuming powerful decision rules.
\item Learn decision functions $g_1,\ldots,g_{K-1}$ for the tree structure learned in Step 1.
\end{enumerate}

\subsection{Learning Tree Structures Greedily}
For simplicity we assume a binary sensor tree with $K$ leaves and $K-1$ internal nodes. Motivated by previous methods~\cite{gao2011active,xu2013cost} we assume that the number of leaves $K$ is small relative to the feature dimension. Our approach identifies a sub-collection of subsets of features ${\cal S}=\{S_1,\ldots,S_K\}$ from training data, and as such, the tree structure, $T$. Assuming arbitrarily powerful decision functions, $g_1,\ldots,g_{K-1}$, effectively implies we can route each example to its optimal subset. Furthermore, assuming that we have access to an oracle classifer, $f_j$'s we can predicted the class on any subset of features. Then the optimization loss of Eq.~\ref{eq.globalerm} associated with the subcollection ${\cal S}$  reduces to:
\begin{align}\label{eq.subsets_min}
L(S_1,\ldots,S_K)= \frac{1}{N}\sum_{i=1}^{N}\min_{j \in \{1,\ldots,K\}}\left[\mathbbm{1}_{f_{s_j}(x_i)\neq y_i} + \alpha \sum_{k \in S_j}c_k\right]
\end{align}

Even in the absence of noise, as noted before the problem of minimizing this loss is NP-hard and motivates greedy strategies. Additionally, we overcome the issue of an oracle classifier by learning it as we grow the tree greedily. While many greedy strategies exist in the literature for related objectives, they are not directly applicable to our setting\footnote{For instance~\cite{cicalese} proposes a submodular surrogate to leverage properties of Wolsey greedy algorithm but their surrogates require discrete sensor measurements.}. Our greedy algorithm is related to \cite{awasthi13}, who learn sparse trees/polynomials in the separable PAC setting and provide statistical and computational guarantees. We adapt their approach to our setting as follows: given a sub-collection of features we expand this subcollection if there is a sensor that can further reduce our loss. If no sensor is found we restart the process and look for a new subset of sensors.

\begin{algorithm}[H]
        \begin{algorithmic}
   \STATE {\bfseries Input:} Number of Subsets $K$
   \STATE {\bfseries Output:} Feature subsets, $s_1,\ldots,s_K$
   \STATE {\bfseries Initialize:} $s_1,\ldots,s_K=\emptyset$
   \FOR{k=1,\ldots,K}
    \STATE $j=\argmin_{j \in s_k^C}\mathcal{L}(s_1,...,s_k\cup j,...,s_K)$
   \WHILE{$L(s_1,...,s_k\cup j,...,s_K)<\mathcal{L}(s_1,...,s_k,...,s_K)$}
   \STATE $s_k=s_k\cup j$
   \STATE $j=\argmin_{j \in s_k^C}\mathcal{L}(s_1,...,s_k\cup j,...,s_K)$
   \ENDWHILE
   \ENDFOR
        \end{algorithmic}
    \caption{Sensor Subset Selection}
    \label{greedy_subset_algorithm}
\end{algorithm}

%
%
%
%
%
%
\textbf{Tree Structure:} Given the set of sensor subsets $s_1,\ldots,s_K$, the problem of choosing a tree structure partitioning between these sensor subsets arises. We propose a hierarchical clustering approach, where subsets are grouped based on the number of common elements. Given a set of feature subsets, the two subsets with the highest number of common elements are grouped together and replaced in the set with the intersection of their elements. This is recursively repeated until only a single subset exists, resulting in a binary tree structure. This can be viewed as a generalization of the cascade approach, as given the set of feature subsets where an additional sensor is added to each previous subset, a cascade structure is always recovered.

Once a tree structure, $T$ is learned we need to populate it with decision functions so that we can generalize to unseen examples. Note that the learned sensor structure provides us with possible choices but does not tell us what choice to make on an unseen example. This motivates the following section.
\subsection{Empirical Risk Problem for a Fixed Tree}


\begin{figure}[ht]
\begin{minipage}{.5\textwidth}
\centering
\includegraphics[height = .5 \linewidth]{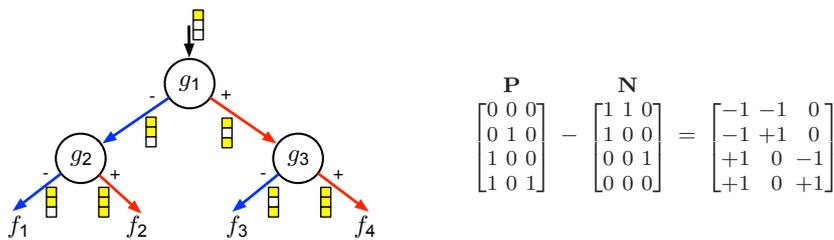}
\end{minipage}%
\begin{minipage}{.5\textwidth}
$$
\begin{array}{c}
\mathbf P \\
\begin{bmatrix}
0 & 0 & 0 \\
0 & 1 & 0 \\
1 & 0 & 0 \\
1 & 0 & 1
\end{bmatrix}
\end{array}
-  \begin{array}{c}
\mathbf N\\
\begin{bmatrix}
1 & 1 & 0 \\
1 & 0 & 0 \\
0 & 0 & 1 \\
0 & 0 & 0
\end{bmatrix}
\end{array} =
\begin{array}{c}
\\
\begin{bmatrix}
-1 & -1 & ~0 \\
 -1 & +1 & ~0 \\
+1 & ~0 & -1 \\
+1 & ~0 & +1
\end{bmatrix}
\end{array}
$$
\end{minipage}
\caption{An example decision system of depth two: node $g_1(x_1)$ selects either to acquire sensor $2$ for a cost $c_2$ or $3$ for a cost $c_3$. Node $g_2(x_1,x_2)$ selects either to stop and classify with sensors $\{1,2\}$ or to acquire 3 for $c_3$ and then stop. Node $g_3(x_1,x_3)$ selects to classify with $\{1,3\}$ or with $\{1,2,3\}$.}
\label{fig:tree}
\end{figure}
We represent our decision system as a binary tree. The binary tree is composed of $K$ leafs and $K-1$ internal nodes. At each internal node, $j=1, \ldots, K-1$, is a binary decision function, $\mbox{sign} [g_j(\bx)] \in \{+1,-1\}$. This function determines which action should be taken for a given example. The binary decisions, $g_j(\bx)$'s, represent actions from the following set: stop and classify with the current set of measurements, or choose which sensor to acquire next. Each leaf node, $k=1,\ldots K$, represents a terminal decision to stop and classify based on the available information\footnote{For notational simplicity, we denote applying a decision node and a leaf classifier as $g_j(\bx)$ and $f_k(\bx)$ respectively.} We assume that the leaf classifiers, $f_j$ are known and fixed\footnote{Note that the classifiers at each leaf, $f_k(\bx) \in \mathcal Y$, can be learned for the $K$ leaves once the tree structure, $T$, has been determined from the previous section.}. However note, the functions implicitly operate only on the sensors that have been acquired along the path to the corresponding node. The objective is to learn the internal decision functions: $g_j(\bx)$'s. We define the system risk:
\begin{equation}\label{eq.point_risk}
R(\bg,\bx,y)=\sum_{k=1}^{K}R_k(f_k,\bx,y)G_k(\bg,\bx)
\end{equation}
Here, $\bg= \{g_1, \ldots g_{K-1}\}$ is the set of decision functions. $R_k(f_k,\bx,y)=\mathds{1}_{f_k(\bx)\neq y}+\alpha \sum_{m \in S_k}c_m$ is the risk of making a decision at a leaf $k$. It consists of two terms: error of the classifier at the leaf and the cost of sensors acquired along the path from the root node to the leaf. $S_k$ is this set of sensors, and $\alpha$ is a parameter that controls trade-off between acquisition cost and classification error.
%
$G_k(\bg,\bx)\in\{0,1\}$ is a binary state variable indicating if $\bx$ is classified at the $k$th leaf. We compactly encode the path from the root to every leaf in terms of internal decisions, $g_j(\bx)$'s, by two auxiliary binary matrices: $\mathbf P$, $\mathbf N$ $\in \{0,1\}^{K \times K-1}$. If $\mathbf P_{k,j}=1$ then, on the path to leaf $k$, a decision node $j$ must be positive: $g_j > 0$. If $\mathbf N_{k,j}=1$ then on the path to leaf $k$, a decision at node $j$ must be negative: $g_j \leq 0$. A $k$th row in $\mathbf P$ and $\mathbf N$ jointly encode a path from the root node to a leaf $k$. The sign pattern for each path is obtained by $\mathbf P-\mathbf N$. For an example refer to Fig. \ref{fig:tree}. Using this path matrix, the state variable can be defined:
$G_k(\bg,\bx)= \prod_{j=1}^{K-1}[\mathds{1}_{g_j(\bx)>0}]^{\mathbf P_{k,j}}[\mathds{1}_{g_j(\bx)\leq 0}]^{\mathbf N_{k,j}}$.

Our goal is to learn decision functions $g_1,\ldots,g_{K-1}$ that minimize the expected system risk:
$\min_{\bg}\mathds{E}_{\cD}\left[R(\bg,\bx,y)\right]$
However, the model $\cD$ is assumed to be unknown and cannot be estimated reliably due to potentially high-dimensional nature of sensor outputs. Instead, we are given a set of $N$ training examples with full sensor measurements, $(\bx_1,y_1),\ldots,(\bx_N,y_N)$. We approximate the expected risk by a sample average over the data and construct the following ERM problem:
\begin{align}
\min_{\bg}\sum_{i=1}^{N}R(\bg,\bx_i,y_i)= \label{eq.erm_prod_ind} \sum_{i=1}^{N}\sum_{k=1}^{K} \overbrace{R_k(f_k,\bx_i,y_i)}^{\text{risk of leaf $k$}}\underbrace{\prod_{j=1}^{K-1}[\mathds{1}_{g_j(\bx_i)>0}]^{\mathbf P_{k,j}}[\mathds{1}_{g_j(\bx_i)\leq 0}]^{\mathbf N_{k,j}}}_{\text{$G_k(\cdot)=$ state of  $\bx_i$ in a tree}} 
\end{align}

Note that by the definition of risk in \eqref{eq.point_risk}, the ERM problem can be viewed as a minimization over a function of indicators with respect to decisions: $g_1(\bx),\allowbreak \ldots, \allowbreak g_{K-1}(\bx)$.



\section{Convex Re-formulation and Solution by Linear Programming}
A popular approach to solving ERM problems is to substitute indicators with convex upper-bounding surrogates, $\phi(z) \geq \indicator{z}$ and then to minimize the resulting surrogate risk. However, such strategy generally leads to a non-convex, multi-linear optimization problem. Previous attempts to solve problems of this form have focused on computationally costly alternating optimization approaches with guarantees of convergence only to a local minimum \cite{trapeznikov:2013b,nips_paper}. Rather than attempting to solve this non-convex surrogate problem, we instead reformulate the indicator empirical risk in \eqref{eq.erm_prod_ind} as a maximization over sums of indicators before introducing convex surrogate. Our approach yields a globally convex upper-bounding surrogate of the empirical loss function. In the next section, we derive this reformulation for a binary tree of depth 2 in Fig. \ref{fig:tree} before generalizing to arbitrary trees.


\subsection{Simple Tree Example}
Consider the decision system shown in Fig. \ref{fig:tree}. The goal is to learn the decision functions $g_1$, $g_{2}$, and $g_{3}$ that minimize the empirical risk \eqref{eq.erm_prod_ind}.

In reformulating the risk, it is useful to define the ''savings" for an example. The \emph{savings}, $\pi_k^i$, for an example $i$, represents the difference between the worst case outcome, $R_{max}$ and the risk $R_k(f_k,\bx_i,y_i)$ for terminating and classifying at the $k$th leaf. The worst case risk is acquiring all sensors and incorrectly classifying: $R_{max}=1+\alpha \sum_{m} c_m$.
\begin{align}
\pi_k^i= R_{max}- R_k(f_k,\bx_i,y_i)
       = \mathds{1}_{f_k(\bx_i)= y_i}+ \alpha \displaystyle \sum_{m \in S_k^C}c_m \label{eq.savings}
\end{align}
Here, $S_k^C$ is the complement set of sensors acquired along the path to leaf $k$ (the sensors not acquired on the path to leaf $k$). Note that the savings do not depend on the decisions, $g_j's$, that we are interested in learning.

For our example, there are only 4 leaf nodes and the state of terminating in a leaf is a encoded by a product of two indicators. For instance, to terminate in Leaf 1, $g_1(\bx_i) \leq 0$ and $g_2(\bx_i) \leq 0$. This empirical risk can be formulated by enumerating over the leaves and their associated risks:
\begin{align} \nonumber
R&(\bg,\bx_i,y_i)= \overbrace{\Big(R_{max}-\pi_1^i\Big)\mathds{1}_{g_1(\bx_i)\leq 0}\mathds{1}_{g_{2}(\bx_i)\leq 0}}^\text{Leaf 1} + \overbrace{\Big(R_{max}-\pi_2^i\Big)\mathds{1}_{g_1(\bx_i)\leq 0}\mathds{1}_{g_{2}(\bx_i)> 0}}^\text{Leaf 2}\\ &+  \underbrace{\Big(R_{max}-\pi_3^i\Big)\mathds{1}_{g_1(\bx_i)> 0}\mathds{1}_{g_{3}(\bx_i)\leq 0}}_\text{Leaf 3} + \underbrace{\Big(R_{max}-\pi_4^i\Big)\mathds{1}_{g_1(\bx_i)> 0}\mathds{1}_{g_{2}(\bx_i)> 0}}_\text{Leaf 4} \label{eq.example_prod}
\end{align}
Directly replacing every $\indicator{z}$ with an upper bounding surrogate such as a hinge loss, $\max[0,1+z] \geq \indicator{z}$, produces a non-convex bilinear objective due the indicator product terms. Bilinear optimization is computationally intractable to solve globally. 

Rather than directly substituting surrogates and solving the non-convex minimization problem, we reformulate the empirical risk with respect to the indicators in the following theorem:
\begin{thm}\label{thm.erm_reform}
The empirical risk in \eqref{eq.example_prod} is equal to \eqref{eq.example_max}.
\begin{align}
R(g_1,g_{2},g_{3},\bx_i,y_i)= R_{max}-\sum_{j=1}^4\pi_j^i&+
\max \Big [(\pi_3^i+\pi_4^i) \mathds{1}_{g_1(\bx_i)\leq 0}+\pi_2^i\mathds{1}_{g_{2}(\bx_i)\leq 0},\notag\\
(\pi_3^i+\pi_4^i)\mathds{1}_{g_1(\bx_i)\leq 0}+\pi_1^i\mathds{1}_{g_{2}(\bx_i)> 0},&
(\pi_1^i+\pi_2^i)\mathds{1}_{g_1(\bx_i)> 0}+\pi_4^i\mathds{1}_{g_{2}(\bx_i)\leq 0},\notag\\
&(\pi_1^i+\pi_2^i)\mathds{1}_{g_1(\bx_i)> 0}+\pi_3^i\mathds{1}_{g_{3}(\bx_i)> 0}\Big]\label{eq.example_max}
\end{align}
\end{thm}
\begin{proof}
Here, we provide a brief sketch of the proof. For full details please refer to the Appendix. We utilize the following two identities: $\indicator{A}\indicator{B} = \min [ \indicator{A}, \indicator{B}]$ and $\indicator{A}=1-\indicator{\bar A}$ and express the risk in \eqref{eq.example_prod} in terms of maximizations:
\begin{align}\label{eq.proofstep}
R \left(g_1,g_{2},g_{3},\bx_i,y_i\right)=R_{max} -\sum_{j=1}^4\pi_j^i
+\pi_1^i \max\left(\mathds{1}_{g_1(\bx_i)> 0},\mathds{1}_{g_{2}(\bx_i)> 0}\right)\notag\\
+\pi_2^i \max\left(\mathds{1}_{g_1(\bx_i)> 0},\mathds{1}_{g_{2}(\bx_i)\leq 0}\right)
 +\pi_3^i \max\left(\mathds{1}_{g_1(\bx_i)\leq 0},\mathds{1}_{g_{3}(\bx_i)> 0}\right)\notag\\
 +\pi_4^i \max\left(\mathds{1}_{g_1(\bx_i)\leq 0},\mathds{1}_{g_{3}(\bx_i)\leq 0}\right) 
\end{align}
Recall that the signs of $g_1,g_2,g_3$ encode a unique path for $\bx_i$. So let us consider sign patterns for each path. For instance, to reach leaf 1, $g_1 \leq 0$ and $g_2 \leq 0$. In this case, by inspection of \eqref{eq.proofstep}, the risk is $(\pi_3^i+\pi_4^i) \indicator{ g_1 (\bx_i) \leq 0} + \pi_2^i \indicator{g_2(\bx_i) \leq 0}+$ constants. This is exactly the first term in the maximization in \eqref{eq.example_max}. We can perform such computation for each leaf (term in the max) in a similar fashion. And due to the interdependencies in \eqref{eq.proofstep}, the term corresponding to a valid path encoding will be the maximizer in \eqref{eq.example_max}.
%
\end{proof}

{\bf Risk Interpretability:} Intuitively, in the reformulated empirical risk in \eqref{eq.example_max}, each term in the maximization encodes a path to one of the $K$ leaves. The largest (active) term correspond to the path induced by the $g_j$'s for an example $\bx_i$. Additionally, the weights on the indicators in \eqref{eq.example_max} represent the \emph{savings lost} if the argument of the indicator is active. For example, if the decision function $g_1(\bx_i)$ is negative, leaves $3$ and $4$ cannot be reached by $\bx_i$, and therefore $\pi_3^i$ and $\pi_4^i$, the savings associated with leaves $3$ and $4$, cannot be realized and are lost.

A distinct advantage of the reformulated risk in \eqref{eq.example_max} arises when replacing indicators with convex upper-bounding surrogates of the form $\phi(z)\geq \mathds{1}_{z\leq0}$. Introducing such surrogates in the original risk in \eqref{eq.example_prod} produces a bilinear function for which a global optimum cannot be efficiently found. In contrast, introducing convex surrogate functions in \eqref{eq.example_max} produces a convex upper-bound for the empirical risk.

\subsection{Extension to Arbitrary Binary Trees}
In this section, we generalize the empirical risk reformulation for any binary tree and present a convex surrogate. Consider a binary tree, $\mathcal T$, composed of $K-1$ internal nodes and $K$ leaves. As defined in \eqref{eq.savings}, each leaf has a corresponding savings $\pi_k^i$ that captures the difference between the worst case risk and the risk of classifying at leaf $k$.

Note that in the previous example, the risk in \eqref{eq.erm_prod_ind} consists of a max of $K$ terms. Each term is a weighted linear combinations of indicators, and each weight corresponds to the \emph{savings lost} if the decision inside the indicator argument is true. For an arbitrary binary tree of $K$ leaves, the risk has an analogous form.

Before stating the result, we define the weights for the linear combination in each term of the max. For an internal node $j$, we denote $C_j^n$ as the set of leaf nodes in a subtree corresponding to a negative decision $g_j(\bx)\leq 0$. And $C_j^p$ is the set of leaf nodes in a subtree corresponding to a positive decision. For instance in Fig. \ref{fig:tree}, $C_1^p=\{Leaf~3,Leaf~4\}$, and in our example \eqref{eq.example_max}, the weight multiplying $\indicator{g_1(\bx_i) \leq 0}$ is the sum of these savings for leaves 3 and 4 (i.e. savings lost if $g_1 \leq 0$). Therefore, sets $C^{p/n}_j$ define which $\pi^i_k$'s contribute to a weight for a decision term.

For a compact representation, recall that the $k$th rows in matrices $\mathbf P$ and $\mathbf N$ define a path to leaf $k$ in terms of $g_1, \ldots, g_{K-1}$, and a non-zero $\mathbf P/\mathbf N_{k,j}$ indicates if $g_j \lessgtr 0$ is on the path to leaf $k$. So for each $\bx_i$ and each leaf $k$, we introduce two positive weight row vectors of length $K-1$:

\begin{align*}
\mathbf w^i_{n,k} =
 \left [\mathbf N_{k,1} \sum_{l \in C_1^p}\pi_l^i,\ldots, \mathbf N_{k,K-1} \sum_{l \in C_{K-1}^p}\pi_l^i \right ]\\
\mathbf w^i_{p,k} = \left [ \mathbf P_{k,1} \sum_{l \in C_1^n} \pi_l^i, \ldots, \mathbf P_{k,K-1} \sum_{l \in C_{K-1}^n} \pi_l^i \right ]
\end{align*}

The $j$th component of $\mathbf w^i_{n,k}$ multiplies $\indicator{g_j(\bx_i) \leq 0}$ in the term corresponding to the $k$th leaf. For instance in our $4$ leaf example in \eqref{eq.example_max}, $\left ( \mathbf w^i_{n,1} \right)_1=\pi^i_3+\pi^i_4$. If $\mathbf P/\mathbf N_{k,j}$ is zero then decision $g_j \gtrless 0$ is not on the path to leaf $k$ and the weight is zero. Using these weight definitions, the empirical risk in \eqref{eq.example_max} can be extended to arbitrary binary trees:
\begin{cor}\label{cor:tree_emp_risk}
The empirical risk of tree $\mathcal T$ is:
\begin{align}\label{eq.erm_tree}
R(\bg,\bx_i,y_i)=R_{max}-\sum_{k=1}^K \pi^i_k \notag\\
+\max_{k \in \{1, \ldots, K\}}  \mathbf w_{p,k}^i \begin{bmatrix} \mathds{1}_{g_1(\bx_i)>0}\\ \vdots \\ \mathds{1}_{g_{K-1}(\bx_i)>0} \end{bmatrix}
+\mathbf w_{n,k}^i \begin{bmatrix}
 \mathds{1}_{g_1(\bx_i)\leq 0}\\ \vdots \\ \mathds{1}_{g_{K-1}(\bx_i)\leq 0}
\end{bmatrix}
\end{align}
\end{cor}
The proof of this corollary is included in the Appendix and follows the same steps as Thm. \ref{thm.erm_reform}.

The empirical risk in \eqref{eq.erm_tree} represents a scan over the paths to each leaf ($k=1,\ldots,K$), and the active term in the maximization corresponds to the leaf to which an observation is assigned by the decision functions $g_1,\ldots,g_{K-1}$. An important observation is that each term in the max in \eqref{eq.erm_tree} is a linear combination of indicators instead of a product as in \eqref{eq.erm_prod_ind}. This transformation enables us to upper-bound each indicator function with a convex surrogate, $\phi(z)$:
$\phi[g_j(\bx)] \geq \indicator{g_j(\bx) >0}~,~\phi[-g_j(\bx)] \geq \indicator{g_j(\bx) \leq 0}$
. And the result is a novel convex upper-bound on the empirical risk in \eqref{eq.erm_tree}. We denote this risk as ${R}_\phi(\bg)$. And the optimization problem over a set of training examples, $\{\bx_i,y_i\}_{i=1}^N$ and a family of decision functions $\mathcal{G}$:
$
\min_{\bg \in \mathcal{G}} \sum_{i=1}^{N}{R}_\phi\left(\bg,\bx_i,y_i\right) \label{eq.erm_min}
$.
%
\subsection{Linear Programming}
There are many valid choices for the surrogate $\phi(z)$. However, if a hinge loss is used as an upper bound and $\mathcal G$ is a family of linear functions of the data then the optimization problem in \eqref{eq.erm_min} becomes a linear program (LP).
\begin{prop}\label{prop.lp_tree}
For $\phi(z)=\max(1-z,0)$ and linear decision functions $g_1,\allowbreak \ldots, \allowbreak g_{K-1}$, the minimization in \eqref{eq.erm_min} is equivalent to the following linear program:

\begin{align}
&\min_{\substack{g_1,\ldots,g_{K-1},\gamma^1,\ldots,\gamma^N \\ \alpha_1^1,\ldots,\alpha_{K-1}^N,\beta_1^1, \ldots, \beta_{K-1}^N}}\sum_{i=1}^{N} \gamma^i  \label{eq.tree_lp}\\
\text{subject to:}\notag\\
&\gamma^i\geq \mathbf w^i_{p,k} \begin{bmatrix}\alpha_1^i\\ \vdots \\ \alpha_{K-1}^i \end{bmatrix} + \mathbf w^i_{n,k} \begin{bmatrix} \beta_1^i \\ \vdots \\ \beta_{K-1}^i \end{bmatrix}, \,\,\,\, \begin{matrix}i \in [N]\\k \in [K]\end{matrix},\notag\\
&\begin{matrix}
1+g_j(\bx_i)\leq \alpha_j^i,\\ 
1-g_j(\bx_i)\leq \beta_j^i,\\
\alpha_j^i\geq 0,\beta_j^i\geq 0,\\
\end{matrix} \,\,\,\,\,\, \begin{array}{l} i \in [N]\\k \in [K-1]\end{array}.\notag
\end{align}
\end{prop}

We introduce the variable $\gamma^i$ for each example $\bx_i$ to convert from a maximization over leaves to a set of linear constraints. Similarly, the maximization within each hinge loss is converted to a set of linear constraints. The variables $\alpha_j^i$ upper-bound the indicator $\mathds{1}_{g_j(x_i)> 0}$ and the variables $\beta_j^i$ upper-bound the indicator $\mathds{1}_{g_j(x_i)\leq 0}$. Additionally, the constant terms in the risk are removed for notational simplicity, as these do not effect the solution to the linear program. For details please refer to Appendix.

{\bf Complexity:} Linear programming is a relatively well-studied problem, with efficient algorithms previously developed. Specifically, for $K$ leaves, $N$ training points, and a maximum feature dimension of $D$, we have $O(KD+KN)$ variables and $O(KN)$ constraints. The state of the art primal-dual methods for LP are fast in practice, with an expected number of iterations $O(\sqrt{n} \log n)$, where $n$ is the number of variables \cite{LP_probabilisty_complexity}.

\section{Experiments}
\begin{table}
\caption{Small dataset descriptions}
\label{table.dataset_info}       
\begin{tabular}{llllll}
\hline\noalign{\smallskip}
Name & Classes & Stage 1 & Stage 2 & Stage 3 & Stage 4\\
\noalign{\smallskip}\hline\noalign{\smallskip}
letter & 26 & Pixel Count & Moments & Edge Features & -\\
landsat & 6 & Band 1 & Band 2 & Band 3 & Band 4\\
SUN Mech. Turk & 16 & Function & Materials & Surf./Spat. Prop.& -\\
Image Seg. & 7 & Location  & Pixel Int. & Color & $\cdots$ \\
\noalign{\smallskip}\hline
\end{tabular}
\end{table}


We demonstrate performance of our proposed approach on two types of real world data sets. First, we demonstrate performance of the LP formulation on examples with fixed structures. In these cases, sensor subsets do not need to be learned and an exhaustive tree can be constructed over all subsets of sensors. We compare the performance of our approach to the alternating optimization scheme presented in \cite{trapeznikov:2013b} applied to the same tree, demonstrating efficiency and performance of our LP formulation. Next, we apply our proposed approach to data sets where dimensionality prevents exhaustive search through feature subsets and both feature subsets and tree structure must be learned along with decision functions in the tree. We show performance on classification data sets presented in \cite{kusner2014feature} and compare performance with Cost Sensitive Trees of Classifiers (CSTC) \cite{xu2013cost} and Approximately Submodular Trees of Classifiers (ASTC) \cite{kusner2014feature}.

For all feature subset classifiers used in our proposed approach and the alternating optimization approach, classifiers are trained using logistic loss on $2^{nd}$-order homogeneous polynomial expanded basis on the entire training set.

\textbf{Learning Decision Functions in Fixed Structures:} Fig. \ref{fig:exp_perf} shows performance of our proposed approach on 5 data sets where few sensors are used and fixed structure can be easily found. For the letter, landsat, and SUN Mechanical Turk data sets \cite{uci_repository,SUN_mechanical_turk}, a tree can be easily constructed using all possible feature subsets. For the image segmentation data set, \cite{uci_repository}, we fix a greedily constructed tree with 8 leaves.

For comparison, we use the alternating minimization approach proposed in \cite{trapeznikov:2013b}. Additionally, we also show performance of a simple myopic strategy for a baseline comparison on these example. The LP approach generally performs comparably to the non-convex alternating optimization approach. Additionally, as shown in Table \ref{table.performance}, the LP is dramatically faster during training time with comparable performance to the alternating training approach.\footnote{All computations were performed on an Intel I5 M430 CPU @ 2.27 GHz with 4 cores.} Note that we do not compare performance to ASTC or CSTC for these experiments. The purpose of these examples is to show the efficacy of the LP for training sequential decision functions independent of feature subsets, classifier design, or tree structure. Both CSTC and ASTC simultaneously learn the feature subsets, classifiers, and decision functions, limiting comparison.

\begin{figure*}[htb!]
\centering
\subfigure[letter]{\includegraphics[width=.49\linewidth]{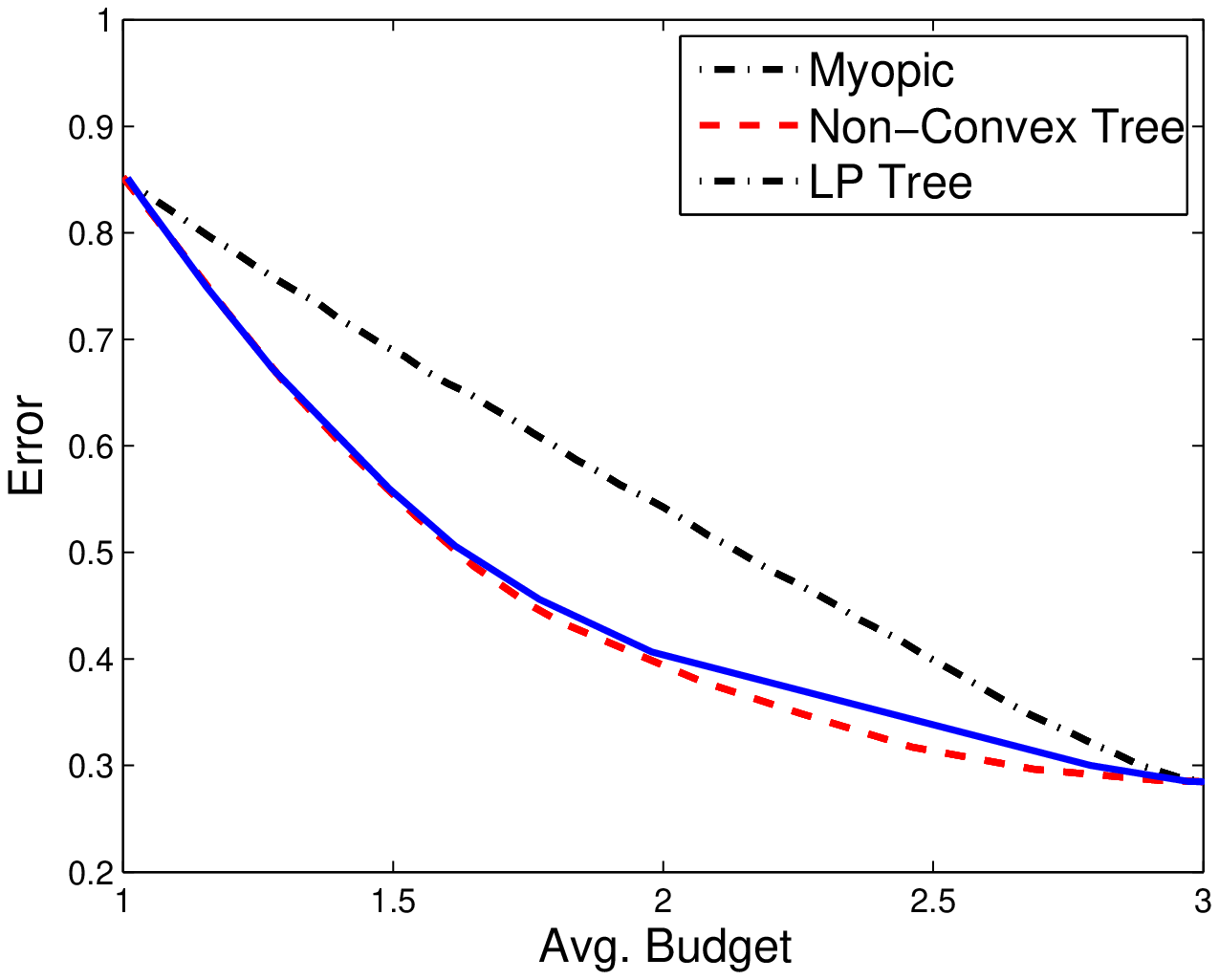}}
\subfigure[landsat]{\includegraphics[width=.49\linewidth]{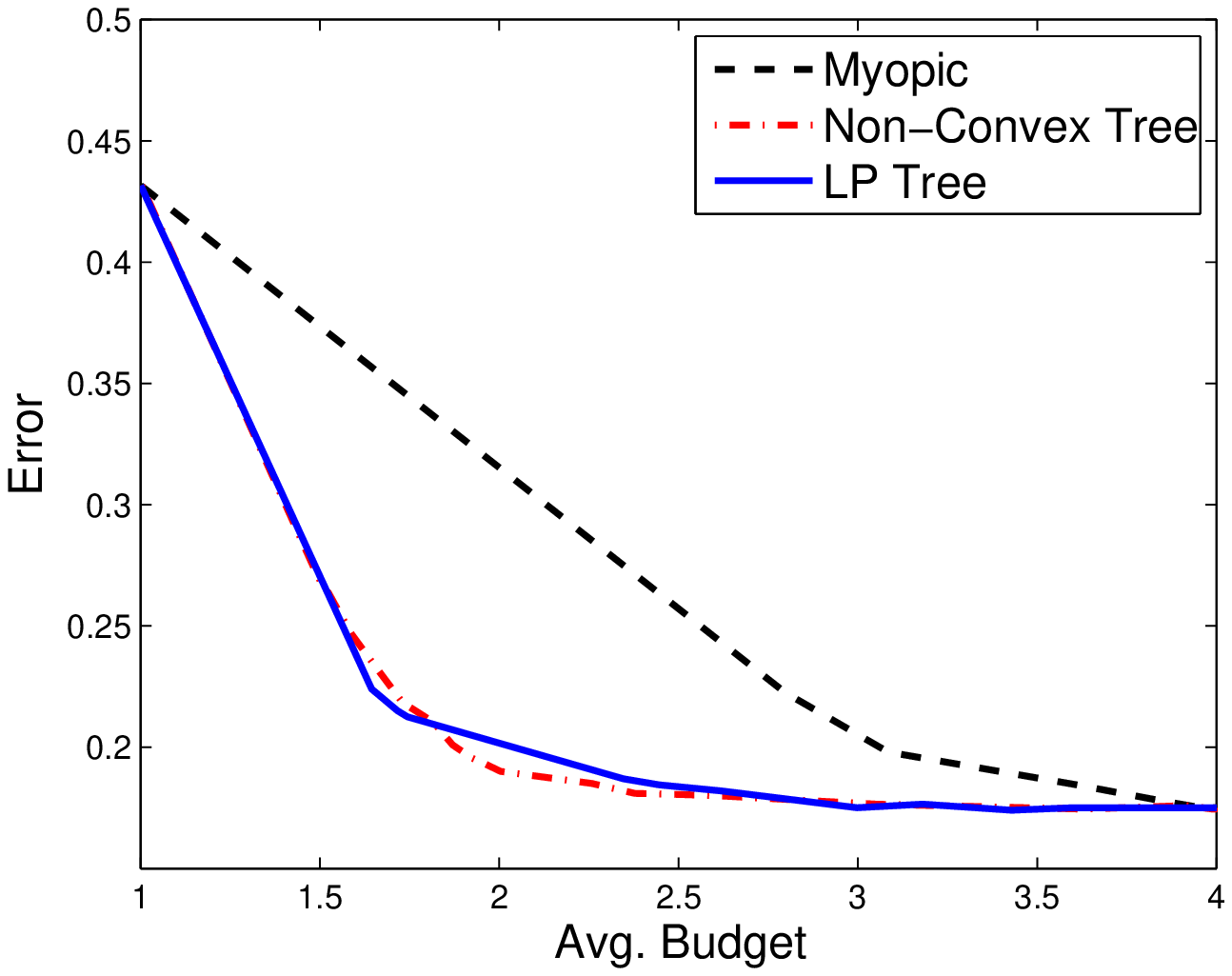}}\\
\subfigure[Image Seg.]{\includegraphics[width=.49\linewidth]{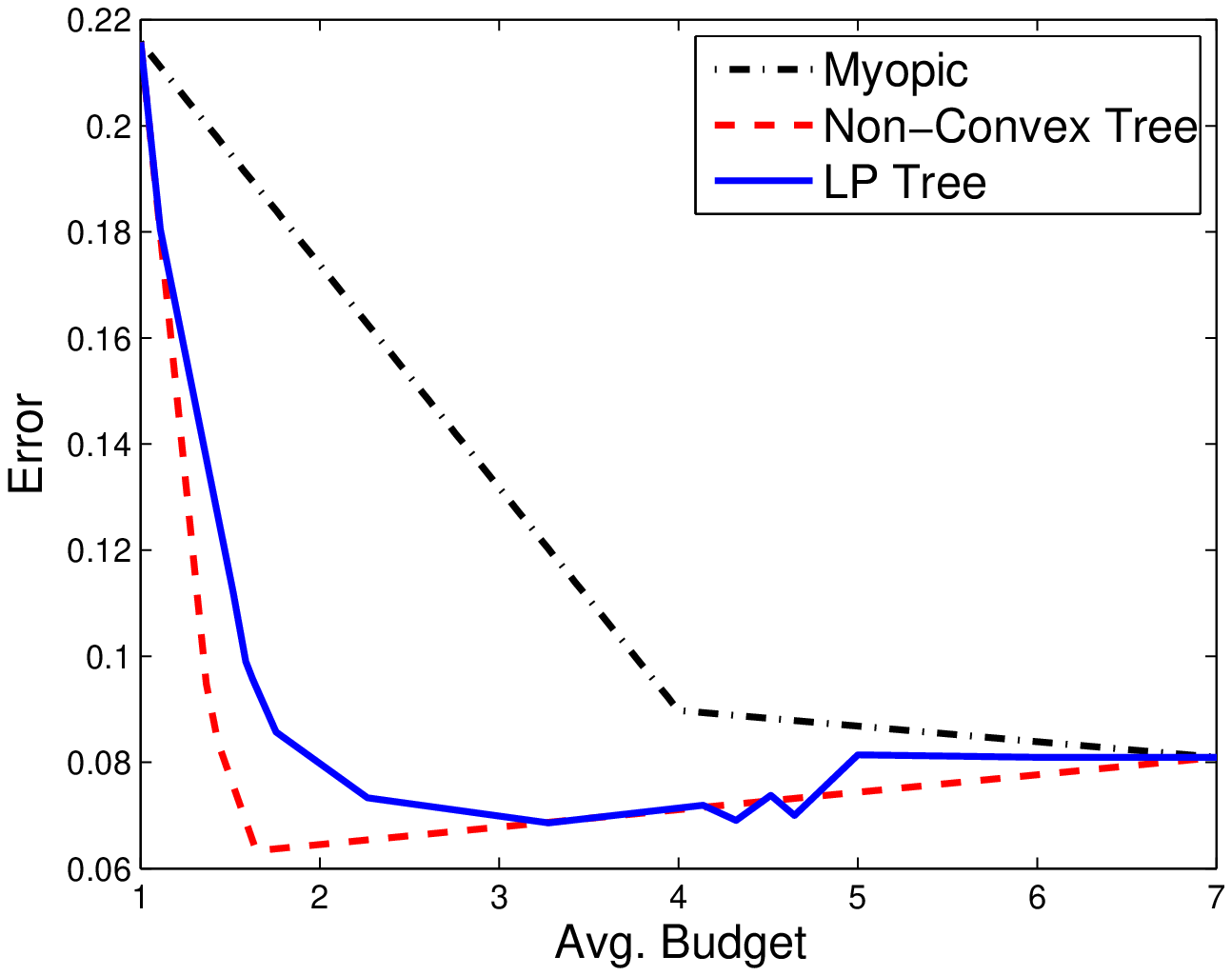}}
\caption{Comparison of error vs. average budget trade-off between a myopic, AM~\cite{trapeznikov:2013b}, our LP method. LP clearly out performs the myopic approach, and generally matches or exceeds the non-convex approach with the added benefit of reduced computational cost, repeatability, and guaranteed convergence.}
\label{fig:exp_perf}
\end{figure*}

\begin{table}
\caption{Average percentage of the budget required to achieve a desired error rate chosen to be close to the error achieved using the entire set of features (approximately $95\%$ of the improvement gained using all features compared the initial features). The percentage of the budget required is with respect to the maximum budget. The training time is the amount of time (in seconds) required to learn a policy for a fixed budget trade-off parameter $\alpha$.}
\label{table.performance}       
\begin{tabular}{lllllll}
\hline\noalign{\smallskip}
Dataset  &\begin{tabular}{c}Target\\ Errors\end{tabular}& Myopic & AM & LP & \begin{tabular}{c}AM Train\\ Time(sec)\end{tabular}& \begin{tabular}{c}LP Train\\ Time(sec)\end{tabular}\\
\noalign{\smallskip}\hline\noalign{\smallskip}
letter &$40\%$& $73\%$ & $48\%$ & $49\%$ &$93.56$ & $57.03$   \\
landsat &$15\%$& $100\%$ & $75\%$ & $75\%$  & $186.0$ &$108.7$ \\
SUN Mech. Turk &$40.4\%$& $99\%$ & $90\%$ & $90\%$  & $2818.9$ & $71.08$\\
Image Seg. &$9\%$& $56\%$ & $21\%$ & $26\%$  & $46.26$ & $16.46 $\\
\noalign{\smallskip}\hline
\end{tabular}
\end{table}

\textbf{Learning Tree Structure:} We next applied our proposed approach to three data sets where learning a tree over all feature subsets is infeasible. For all three datasets, we learn the set of subsets as described in Section \ref{sec.struct}, with a total of 16 subsets of features (leaves of the tree) used. For the MiniBooNE and forest data sets, the proposed approach outperforms CSTC, with performance exceeding ASTC for the forest data set and matching ASTC for MiniBooNE. On the CIFAR dataset, the proposed LP approach matches CSTC when using 50 features, but otherwise is generally outperformed by both CSTC and ASTC. We attribute this to the limited complexity ($2^{nd}$ order homogeneous polynomial) of the classification functions, which requires more features to gain flexibility to accurately partition the data.

\begin{figure*}[htb!]
\centering
\subfigure[Forest]{\includegraphics[width=.49\linewidth]{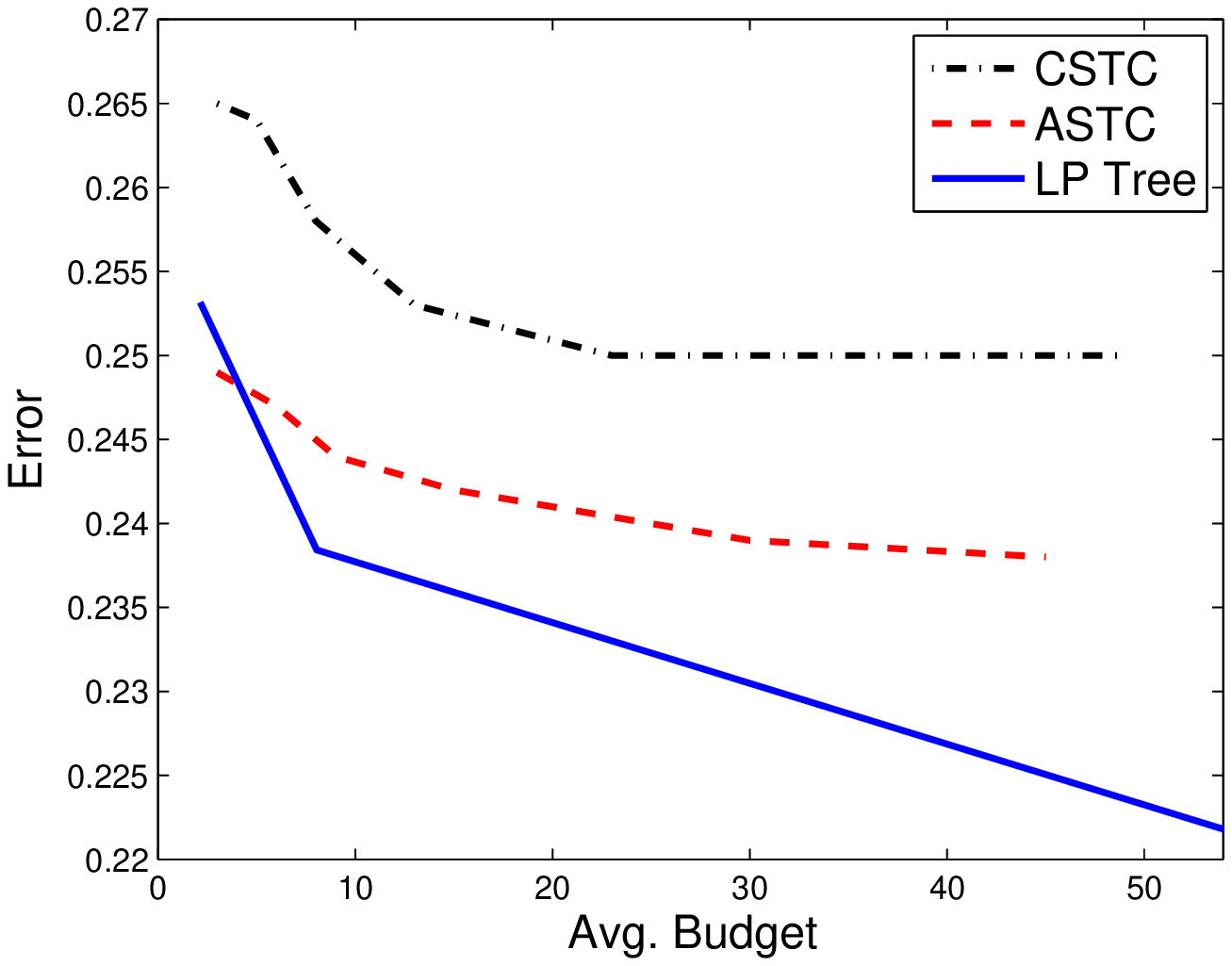}}
\subfigure[MiniBooNE]{\includegraphics[width=.49\linewidth]{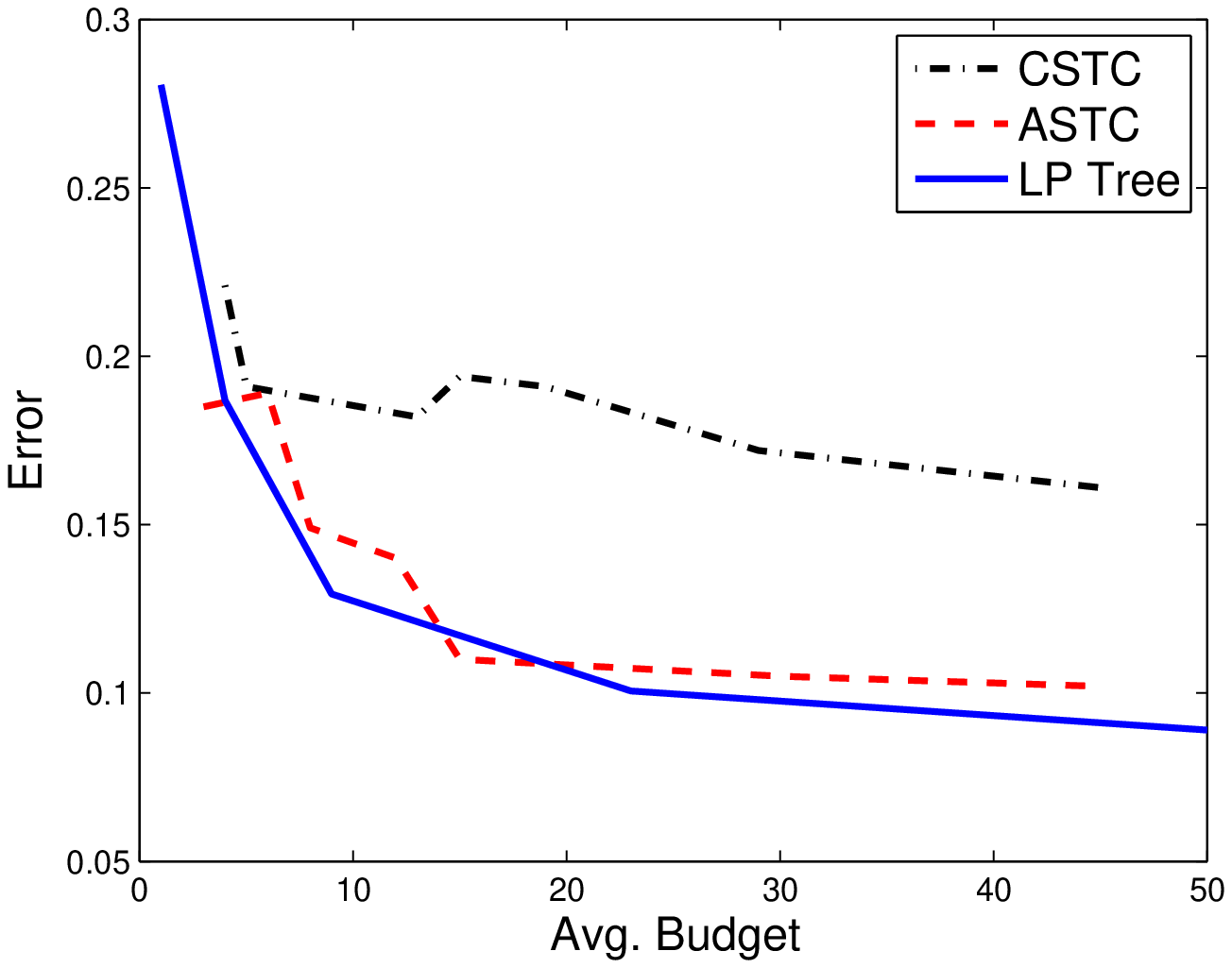}}\\
\subfigure[CIFAR]{\includegraphics[width=.49\linewidth]{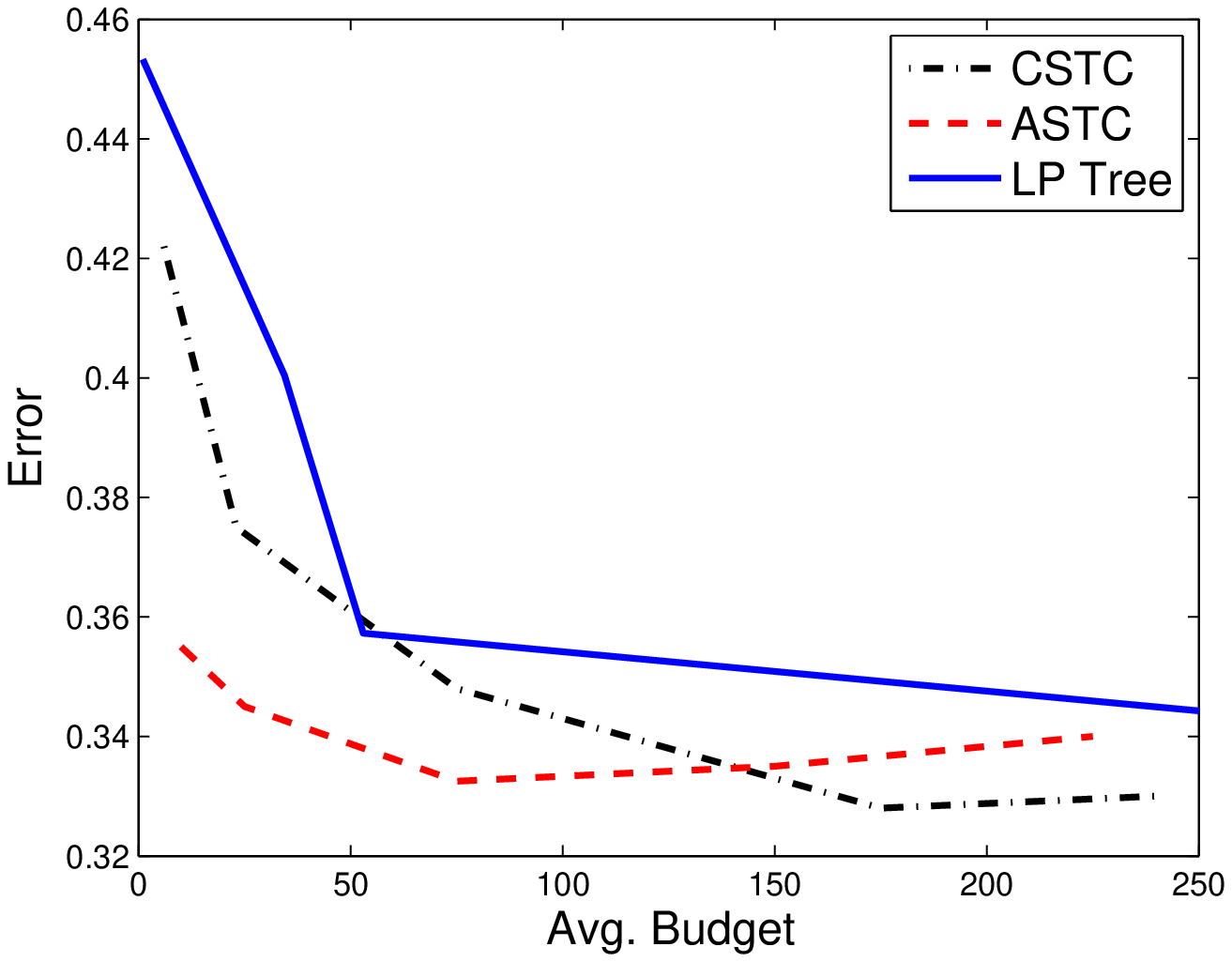}}
\caption{Plot of our proposed approach (LP tree) and CSTC on three real world data sets. On all three data sets, LP tree generally outperforms CSTC, producing high levels of classification accuracy with very low budgets.}
\label{fig:exp_perf2}
\end{figure*}


\bibliographystyle{spbasic}      
\section{Conclusion}

\appendix
\section{Proof of Theorem \ref{thm.erm_reform}}
The product of indicators can be expressed as a minimization over the indicators, allowing the empirical loss to be expressed:
\begin{align}
R & \left(g_1,g_{21},g_{22},x_i,y_i\right)=\Bigg(R_{max} -\pi_1^i \min\left(\mathds{1}_{g_1(x_i)\leq 0},\mathds{1}_{g_{21}(x_i)\leq 0}\right)\notag\\
&-\pi_2^i \min\left(\mathds{1}_{g_1(x_i)\leq 0},\mathds{1}_{g_{21}(x_i)> 0}\right) -\pi_3^i \min\left(\mathds{1}_{g_1(x_i)> 0},\mathds{1}_{g_{22}(x_i)\leq 0}\right)\notag\\
&-\pi_4^i \min\left(\mathds{1}_{g_1(x_i)> 0},\mathds{1}_{g_{22}(x_i)> 0}\right)\Bigg).\notag
\end{align}
By swapping the inequalities in the arguments of the indicator functions, the minimization functions can be converted to maximization functions:
\begin{align}
R & \left(g_1,g_{21},g_{22},x_i,y_i\right)=R_{max}
+\pi_1^i \max\left(\mathds{1}_{g_1(x_i)> 0},\mathds{1}_{g_{21}(x_i)> 0}\right)-\pi_1^i\notag\\
& +\pi_2^i \max\left(\mathds{1}_{g_1(x_i)> 0},\mathds{1}_{g_{21}(x_i)\leq 0}\right)-\pi_2^i +\pi_3^i \max\left(\mathds{1}_{g_1(x_i)\leq 0},\mathds{1}_{g_{22}(x_i)> 0}\right)\notag\\
& -\pi_3^i+\pi_4^i \max\left(\mathds{1}_{g_1(x_i)\leq 0},\mathds{1}_{g_{22}(x_i)\leq 0}\right)  -\pi_4^i\Bigg).\notag
\end{align}
Note that due to the dependence of the indicators, there will always be 3 maximization terms equal to 1 and 1 maximization term equal to zero. As a result, the sum of maximizations can be expressed as a maximization over the 4 possible combinations, yielding the expression:
\begin{align}
R & \left(g_1,g_{21},g_{22},x_i,y_i\right)=\Bigg(R_{max}-\pi_1^i -\pi_2^i -\pi_3^i -\pi_4^i \notag\\
&+\max \Big((\pi_3^i+\pi_4^i)\mathds{1}_{g_1(x_i)\leq 0}+\pi_2^i\mathds{1}_{g_{21}(x_i)\leq 0},(\pi_3^i+\pi_4^i)\mathds{1}_{g_1(x_i)\leq 0}+\pi_1^i\mathds{1}_{g_{21}(x_i)> 0},\notag\\
&(\pi_1^i+\pi_2^i)\mathds{1}_{g_1(x_i)> 0}+\pi_4^i\mathds{1}_{g_{21}(x_i)\leq 0},(\pi_1^i+\pi_2^i)\mathds{1}_{g_1(x_i)> 0}+\pi_3^i\mathds{1}_{g_{21}(x_i)> 0}\Big)\Bigg).\notag
\end{align}

\section{Proof of Corollary \ref{cor:tree_emp_risk}}
The product of indicators over an arbitrary binary tree is given by:
\begin{align*}
&R(\bg,\bx_i,y_i)=\\
&\sum_{k=1}^{K} \overbrace{R_k(f_k,\bx_i,y_i)}^{\text{risk of leaf $k$}}\underbrace{\prod_{j=1}^{K-1}[\mathds{1}_{g_j(\bx_i)>0}]^{\mathbf P_{k,j}}[\mathds{1}_{g_j(\bx_i)\leq 0}]^{\mathbf N_{k,j}}.}_{\text{state of $G_k(\cdot)=$ $\bx_i$ in a tree}}
\end{align*}
Converting the product into a minimization over indicators, the function can be rewritten:
\begin{align*}
R(\bg,\bx_i,y_i)=\sum_{k=1}^{K} \left(R_{max}-\pi_k^i\right)\min_{j \in \{1,\ldots,K-1\}}\left([\mathds{1}_{g_j(\bx_i)>0}]^{\mathbf P_{k,j}},[\mathds{1}_{g_j(\bx_i)\leq 0}]^{\mathbf N_{k,j}}\right)
\end{align*}
and using the identity $\mathds{1}_{A}=1-\mathds{1}_{\bar{A}}$, this can be converted to the maximization:
\begin{align*}
R(\bg,\bx_i,y_i)=R_{max}-\sum_{k=1}^{K}\pi_k^i+\sum_{k=1}^{K} \pi_k^i\max_{j \in \{1,\ldots,K-1\}}\left([\mathds{1}_{g_j(\bx_i)\leq0}]^{\mathbf P_{k,j}},[\mathds{1}_{g_j(\bx_i)> 0}]^{\mathbf N_{k,j}}\right).
\end{align*}
As in the 2-region case, the dependence of the indicators always results in $K-1$ maximization terms equal to $1$ and 1 maximization term equal to $0$. By examination, the sum of maximization functions can be expressed as a single maximization over the paths of the leaves, resulting in a loss shown in \eqref{eq.erm_tree}.

\section{Additional Explanation of Prop. 4.1}
The linear program of Prop. \ref{prop.lp_tree} is constructed by replacing the indicators with hinge-losses of the appropriate signs:
\begin{align}
&\min_{\substack{g_1,\ldots,g_{K-1},\gamma^1,\ldots,\gamma^N \\ \alpha_1^1,\ldots,\alpha_{K-1}^N,\beta_1^1, \ldots, \beta_{K-1}^N}}\sum_{i=1}^{N} \gamma^i \\
\text{subject to:}\notag\\
&\gamma^i\geq \mathbf w^i_{p,k} \begin{bmatrix}\alpha_1^i\\ \vdots \\ \alpha_{K-1}^i \end{bmatrix} + \mathbf w^i_{n,k} \begin{bmatrix} \beta_1^i \\ \vdots \\ \beta_{K-1}^i \end{bmatrix}, \,\,\,\, \begin{matrix}i \in [N]\\k \in [K]\end{matrix},\notag\\
&\begin{matrix}
1+g_j(\bx_i)\leq \alpha_j^i,\\ 
1-g_j(\bx_i)\leq \beta_j^i,\\
\alpha_j^i\geq 0,\beta_j^i\geq 0,\\
\end{matrix} \,\,\,\,\,\, \begin{array}{l} i \in [N]\\k \in [K-1]\end{array}.\notag
\end{align}
Note that the linear program arises based on the fact that any maximization can be converted to a linear constraint with the introduction of a new variable. The maximization in the objective for each observation is replaced by the introduction of the variable $\gamma^i$ and the first constraint. The maximization functions in the hinge losses are replaced by the remaining constraints, introducing the variables $\alpha_j^i=\max(1+g_j(\bx_i),0)$ and $\beta_j^i=\max(1-g_j(\bx_i),0)$, respectively.
 
\bibliography{bibl}   


\end{document}